\documentclass{article}

\usepackage[preprint, nonatbib]{custom_style}

\usepackage{amsmath}
\usepackage{amssymb}
\usepackage{mathtools}
\usepackage{algorithm}
\usepackage{algpseudocode}
\usepackage{booktabs}
\usepackage{graphicx}
\usepackage{xcolor}
\usepackage[english]{babel}

\usepackage{caption}
\usepackage{subcaption}
\usepackage{makecell}

\usepackage{wrapfig}
\usepackage{amsthm}
\usepackage{enumitem}
\usepackage{accents}
\usepackage[numbers]{natbib}

\usepackage[colorlinks=true, linkcolor=niceblue]{hyperref}
\usepackage[capitalize,noabbrev]{cleveref}

\newcommand{\smallsection}[1]{\textbf{#1.~~~~}}
\definecolor{niceblue}{rgb}{0.10, 0.14, 0.76} 

\definecolor{nicered}{rgb}{0.70, 0.0, 0.0} 

\AtBeginDocument{%
\hypersetup{
    citecolor=niceblue,
    linkcolor=red,   
    urlcolor=niceblue,
    linktoc=none}
}

\title{Optimal Embedding Learning Rate in LLMs: \\
The Effect of Vocabulary Size}
\author{%
  Soufiane Hayou\thanks{Corresponding author: \texttt{hayou@berkeley.edu}}\\
  Simons Institute\\
  UC Berkeley \\
  \And
  Liyuan Liu \\
  Microsoft Research\\
}

\date{}

\newcommand{\reals}{\mathbb{R}}
\newcommand{\normal}{\mathcal{N}}
\newcommand{\data}{\mathcal{D}}
\newcommand{\Sign}{\mathcal{S}}
\newcommand{\E}{\mathbb{E}}
\newcommand{\bigO}{\mathcal{O}}

\newcommand\munderbar[1]{%
  \underaccent{\bar}{#1}}

\newtheorem{thm}{Theorem}
\newtheorem{lemma}{Lemma}

\newtheorem{assumption}{Assumption}

\begin{document}

\maketitle

\begin{abstract}
Pretraining large language models is a costly process. To make this process more efficient, several methods have been proposed to optimize model architecture/parametrization and hardware use. On the parametrization side, $\mu P$ (Maximal Update Parametrization) parametrizes model weights and learning rate (LR) in a way that makes hyperparameters (HPs) transferable with width (embedding dimension): HPs can be tuned for a small model and used for larger models without additional tuning. While $\mu$P showed impressive results in practice, recent empirical studies have reported conflicting observations when applied to LLMs. One limitation of the theory behind $\mu$P is the fact that input dimension (vocabulary size in LLMs) is considered fixed when taking the width to infinity. This is unrealistic since vocabulary size is generally much larger than width in practice. In this work, we provide a theoretical analysis of the effect of vocabulary size on training dynamics, and subsequently show that as vocabulary size increases, the training dynamics \emph{interpolate between the $\mu$P regime and another regime that we call Large Vocab (LV) Regime}, where optimal scaling rules are different from those predicted by $\mu$P. Our analysis reveals that in the LV regime, the optimal embedding LR to hidden LR ratio should roughly scale as $\Theta(\sqrt{width})$, surprisingly close to the empirical findings previously reported in the literature, and different from the $\Theta(width)$ ratio predicted by $\mu$P. We conduct several experiments to validate our theory, and pretrain a 1B model from scratch to show the benefit of our suggested scaling rule for the embedding LR.
\end{abstract}

\section{Introduction}
Large Language Models (LLMs) require intensive pretraining on massive amounts of data before progressing to the post-training stage where they are finetuned for specific applications. During pretraining, model weights are updated with optimizers such as Adam \citep{kingma2017adammethodstochasticoptimization} to minimize the next-token prediction loss. This process involves tuning several hyperparameters (HPs) in the model such as initialization, learning rate, batch size -- a process that becomes significantly expensive as model size increases because optimal HPs are \emph{highly sensitive to scale}. 

To address this challenge, researchers have studied how optimal HPs change with scale. One of the earliest works in this direction can be attributed to \citet{Neal1996} where the author studied how initialization should be scaled as network width increases. Subsequent research has examined the impact of the activation functions, learning rates, batch sizes, and other factors \citep{hayou19activation, hayou2022on, yang2022tensor, zhang2025doescriticalbatchsize}. Additional studies have studied how these HPs should scale with depth \cite{deepinfoprop2017, hayou2021stableresnet, yang2023tensorprogramsvifeature, bordelon2023depthwisehyperparametertransferresidual}.

\paragraph{Maximal Update Parametrization \citep{yang2022tensor} and related works.} $\mu$P  provides scaling rules with respect to width for initialization and learning rate in general neural network architecture. For instance, with Adam optimizer, $\mu$P parametrizes the learning rate for the hidden layers as $\eta \times d^{-1}$ where $d$ is model width (embedding dimension). With this parametrization, the constant $\eta$ can be tuned for small model and used for larger models without further tuning (HP transfer).

However, when used for LLM pretraining, conflicting empirical results regarding the benefits of $\mu$P have been reported in the literature \citep{falcon2023falconseriesopenlanguage, blake2025umupunitscaledmaximalupdate, lingle2025empiricalstudymuplearning, jordan2024muon, everett2024scalingexponentsparameterizationsoptimizers}. A potential explanation for these conflicting observations is that vocabulary size was not incorporated in the theoretical analysis. Indeed, a significant limitation of the theory behind $\mu$P (Tensor Programs) is that it assumes fixed vocabulary size while studying feature learning in the large width limit. This assumption is unrealistic, as in practice, vocabulary size is often much larger than width (see e.g. \cite{2024llama3herdmodels, abdin2024phi, gemmateam2025gemma3technicalreport}). Consequently, it is unclear whether $\mu$P scaling rules remain optimal in the case of LLMs. Intuitively, vocabulary size would most significantly impact embedding and projection layers, and it should be expected that optimal learning rates for these two modules would shift as vocabulary size increases. This is the focus of our paper, where we develop a theoretical framework to study the impact of vocabulary size on training dynamics. 

Our contributions are three-fold:
\begin{itemize}
    \item Using a simple theoretical framework, we provide a rigorous analysis of the impact of vocabulary size on feature learning dynamics, and demonstrate the existence of two regimes: (i) the $\mu$P regime, where vocabulary size is fixed and only width (embedding dimension) $d$ increases to infinity, and 2) Large Vocabulary (LV) regime where vocabulary size also increases, in which case the $\mu$P scaling rules for embedding layer LR are suboptimal. 
    
    \item Our theory suggests that the ratio of embedding layer LR (LR$_{emb}$) to hidden layers LR (LR$_{hidden}$) should scale roughly as  LR$_{emb}$/LR$_{hidden}=\Theta_d(\sqrt{d})$, in contrast to $\mu$P prediction of $\Theta_d(d)$ ratio. As vocabulary size increases, the training dynamics interpolate between the $\mu$P regime and the LV regime.
    
    \item We hypothesize that the LV regime is more adapted to modern LLM pretraining and conduct extensive experiments to validate our theoretical findings. Notably, we pretrain a 1B model from scratch to show that setting LR$_{emb}$/LR$_{hidden}\approx\sqrt{d}$ leads to significant improvement in large scale pretraining.
\end{itemize}

\section{Neural Parametrizations and Hyperparameter Transfer}
As we scale neural networks, we must adapt hyperparameters in a scale-dependent manner. When model width $d$ increases, both initialization variance and learning rate should be adapted to avoid numerical instabilities and ensure efficient learning. For instance, the variance of initialization weights in hidden layers should scale as $d^{-1}$ to prevent excessively large activations as model width increase (e.g. He init \cite{he2015delvingdeeprectifierssurpassing}). To derive such scaling rules, a principled approach consists of analyzing statistical properties of the training dynamics as $d$ grows, then adjusting initialization, learning rate, and architecture to achieve desirable properties in the limit $d \to \infty$ \citep{deepinfoprop2017, hayou19activation, yang2019scaling}. 

$\mu$P provides scaling exponents for initialization and learning rate. The goal is to achieve ``maximal'' feature learning and avoid lazy training when width is significantly large.\footnote{Lazy training \citep{chizat2020lazytrainingdifferentiableprogramming} refers to a phenomenon where feature learning becomes suboptimal as width increases. This is primarily associated with the Neural Tangent Kernel \citep{jacot2020neuraltangentkernelconvergence}.} One of the key benefits of $\mu$P is HP transfer, where optimal HPs transfer with width: we can tune HPs on a small model and apply the same HPs to larger models without further tuning. This is particularly valuable when training large models, as it significantly reduces training costs. While $\mu$P was established for general neural architectures,\footnote{Under the assumption that all neural computations can be represented as tensor programs.} our focus in this paper is specifically on LLMs, and particularly decoder-only transformer architectures. 

A transformer model of depth $L$ and width $d$ is given by 
\begin{equation}
\begin{cases}
Y_{emb}(X) = X \,E,\\
Y_l(X) = Y_{l-1}(x) + \mathcal{F}(Y_{l-1}(X), \theta_l), l\in[1:L],\\
Y_{proj}(X) = Y_{L}(X) W,
\end{cases}
\end{equation}
where $E \in \reals^{m\times d}$ is the embedding matrix (vocabulary size $m$), $X \in \reals^{T \times m}$ is a tokenized input sequence of length $T$ (each row consists of a one-hot vector), $W \in \reals^{d \times m}$ is the projection matrix, and $\mathcal{F}(.)$ is a mapping that consists of Attention and MLP blocks with weights $\theta_l$. For LLMs, the transformer is usually trained to minimize the next token prediction loss, which extracts the last row of $Y_{proj}(X)$ and compares it against ground truth next token.

\paragraph{Notation.} Hereafter, $d$ will always denote model width. As $d$ grows, given sequences $c_d \in \reals$ and $b_d \in \reals^+$, we write $c_d = \bigO(b_d)$ to refer to $c_d < \kappa d_b$  for some constant $\kappa > 0$. We write $c_d = \Theta(b_d)$ if we have $\kappa_1 b_d\leq c_d \leq \kappa_2 b_d$ for some $\kappa_1, \kappa_2 >0$. For vector sequences $c_d = (c_d^i)_{1 \leq i \leq k} \in \reals^k$ (for some $k >0$), we write $c_d = \bigO(b_d)$ when $c_d^i = \bigO(b_d^i)$ for all $i \in [k]$, and same holds for other asymptotic notation. Finally, when the sequence $c_d$ is a vector of random variables, asymptotics are defined in the sense of the second moment ($L_2$ norm). For a vector $z\in \reals^d$, $\|z\| = \left(\sum_{i=1}^d z_i^2\right)^{1/2}$ refers to the euclidean norm.

\paragraph{Neural parametrization.} In the context of model scaling, a parametrization specifies how each HP in the model should depend on the scalable dimension (scaling exponents). For instance, when scaling width $d$, initialization variance conventionally scale as $d^{-1}$ \citep{he2015delvingdeeprectifierssurpassing}. $\mu$P suggests similar scaling rules for learning rate. While various HPs are sensitive to scale, including batch size and activation function, our focus in this work is primarily on learning rate adaptation.

\paragraph{Mechanisms of $\mu$P.} The idea behind $\mu$P originates from the infinite-width theory of neural networks. Specifically, the literature on the neural tangent kernel \citep{jacot2020neuraltangentkernelconvergence} showed that certain parameterizations lead to a kernel regime in the infinite-width limit. In other words, increasing model width with such parametrizations leads to subopotimal for feature learning. The authors of $\mu$P reverse-engineered this problem by asking: how can we maximize feature learning in the infinite-width limit? Additional details on infinite-width theory of neural networks are provided in \cref{app:add_details_infinite_width}.

Hereafter, we use the subscript $t$ to denote the optimizer step. In the context of $\mu$P, feature learning is measured by the change in features after one step, and we say that we have ``maximal'' feature learning if the following holds
$$
\Delta_t Y := Y_{t+1} - Y_{t} = \Theta_d(1),
$$
for any neuron $Y$ in the model, e.g. $Y_l(X)_{ij}$ ($j$-th neuron in the $i$-th hidden vector in the sequence). $\mu$P was derived by enforcing this asymptotic behavior. We summarize $\mu$P rules for initialization and LR (Adam optimizer) in \Cref{tab:mup}.
\begin{table}[ht]
    \centering
    \caption{$\mu$P scaling rules for initialization and learning rate.}
    \begin{tabular}{c|c|c|c}
    \Xhline{2\arrayrulewidth}
        &  Embedding weights  &  Output weights & Hidden weights \\\hline
       Init. Var. & $1$ & $d^{-2}$ & $d^{-1}$ \\\hline
       LR (Adam)  & $\eta$ & $\eta d^{-1}$ & $\eta d^{-1}$\\\Xhline{2\arrayrulewidth}
    \end{tabular}
    \label{tab:mup}
\end{table}

The constant $\eta>0$ is a tunable HP that does not depend on width $d$.\footnote{Think of $\eta$ as the constant in $\Theta_d(.)$.}. As we scale $d$, $\mu$P suggests that the optimal LR for hidden and output weights should be expected to scale as $d^{-1}$, while the optimal LR for embedding layer should remain approximately constant.

In its ``raw'' form, $\mu$P scaling exponents are given by asymptotic results such as LR$_{hidden}=\Theta_d(d^{-1})$ and LR$_{emb} = \Theta_d(1)$ which suggests using different constants $\eta$'s for each cell in \Cref{tab:mup}.\footnote{The notation $a_d = \Theta_d(b_d)$ means that roughly $a_d$ is of the same order as $b_d$ when $d$ is large.} However, since the cost of tuning multiple HPs grows exponentially with their number, such tuning process becomes costly even for relatively small models. Therefore practitioners usually use the same constant $\eta$ across layers \citep{lingle2025empiricalstudymuplearning}. This is also true for standard parametrization (SP) where no width exponents are included in the learning rate.

\begin{wrapfigure}{r}{0.28\textwidth}
  \begin{center}
    \includegraphics[width=0.28\textwidth]{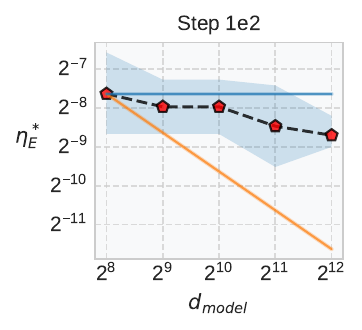}\\
    \vspace{-0.2cm}
    \includegraphics[width=0.28\textwidth]{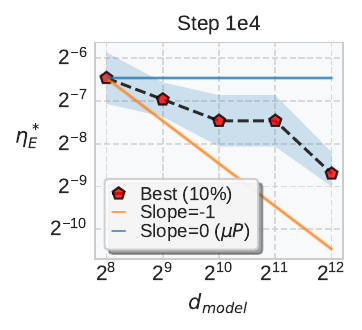}
  \end{center}
  \caption{\small{Optimal Embedding LR vs model width. Transformer model with vocab size 32768, trained on Wikitext2 for 1e4 steps with Adam.}}\label{fig:no_hp_transfer_with_mup}
\end{wrapfigure}
\paragraph{Limitations of $\mu$P for LLMs.} Recent empirical work \citep{blake2025umupunitscaledmaximalupdate} has shown that while the $\Theta_d(d^{-1})$ scaling effectively leads to HP transfer for hidden/output LRs, the optimal embedding LR seems to shift as width $d$ increases, contradicting $mu$P's predictions of a constant optimal embedding LR across scales. In another work \citep{everett2024scalingexponentsparameterizationsoptimizers}, the authors showed that in the case of LLMs different parametrizations can match and even outperform $\mu$P in terms of the quality of HP transfer, including standard parametrization, provided learning rate exponents are corrected to account for width growth. This contradicts $\mu$P's intuition that maximal feature learning is the only way to achieve HP transfer.

Understanding this discrepancy requires examining the fundamental limitations of $\mu$P, especially when applied to LLMs. In this work, we show that two critical factors contribute to this discrepancy:
\begin{itemize}[leftmargin=*]
    \item \underline{Vocabulary size:} The $\mu$P framework assumes fixed input and output dimensions while allowing only width to go to infinity. Consequently, the constants in $\Theta_d(.)$ implicity depend on the input and output dimensions. In LLMs, vocabulary size is typically much larger than width, and larger models generally require larger vocabularies as empirically shown in \citep{tao2024scalinglawsvocabularylarger}.
    \item \underline{Data-specific characteristics}: An important aspect of the transformer architecture is embedding layer, which functions as a lookup table rather than a conventional linear layer. Although $\mu$P predicts $\Theta_d(1)$ scaling for embedding LR, the lookup nature of the embedding layer creates differential update patterns: embedding vectors corresponding to frequent tokens receive significantly more informative updates than those corresponding to less frequent tokens. $\mu$P does not incorporate this aspect in the analysis leading to the stated scaling exponents.
\end{itemize}

Our work addresses both factors through an intuitive analytical framework. We demonstrate how incorporating vocabulary scaling and data-specific characteristics  yield different scaling behavior for optimal embedding LR. Specifically, we establish that Adam-like optimizers that use element-wise normalization lead to an interesting phenomenon: 
\begin{center}
\emph{As vocabulary size increases, the training dynamics interpolate between the $\mu$P regime and a Large Vocabulary (LV) regime where optimal LRs satisfies roughly LR$_{emb}/$LR$_{hidden}=\Theta(d^{-1/2})$.}
\end{center}

Intuitively, the change in the scaling exponents results from the effect of vocabulary size on the alignment between weight updates and incoming activations during training (see next section for an in-depth discussion). 

We trained a decoder-only transformer trained with Adam for $1e4$ steps on Wikitext2 \citep{merity2016wikitext}. We use a large vocab size of 32768 and width $d \in \{256, 512, 1024, 2048, 4096\}$ and sweep over LR$_{emb}$, LR$_{hidden}$, LR$_{out}$, Embedding Init, Hidden Init, and Output Init (See \cref{app:emp} for more details about the experimental setup).  Interestingly, we found that the optimal performance is generally achieved by using the same initialization and learning rate for hidden layers and output layer which contradicts $\mu$P's predictions, and confirms some findings in \cite{everett2024scalingexponentsparameterizationsoptimizers} showing that SP can outperform $\mu$P is some cases.

\cref{fig:no_hp_transfer_with_mup} shows the optimal LR$_{emb}$ as a function of width $d$. The optimal embedding LR seems to decrease with model width at a sublinear rate, contradicting $\mu$P predictions of $\Theta_d(1)$ behavior and confirming the $d^{-1/2}$ scaling previously reported in \citep{blake2025umupunitscaledmaximalupdate}.

\section{Theoretical Analysis}\label{sec:theory}
In this section, we provide an intuitive explanation of the patterns observed in \cref{fig:no_hp_transfer_with_mup}. We conduct an analysis of the infinite-width training dynamics in a controlled setting. We use a linear neural network for simplification, readability, and tractability of the analysis. 

Consider the embedding matrix $
E \in \reals^{m \times d}$, 
where $m$ is the vocabulary size, and $d$ is the width (embedding dim). 
In LLMs, the embedding layer acts as a simple lookup operation. Specifically, given a token with index $i$ in the vocabulary, the forward operation simply returns $E_i$, the $i^{th}$ row in the embedding matrix $E$. Updates to the rows of $E$ are given by
$$
E_i \leftarrow E_i - \eta A(G_i),
$$
where $A(G_i)$ is the processed gradient (e.g. Adam) of the raw gradient given by 
$$
G_i = |B|^{-1} \sum_{x\in B} \nabla_{E_i}\ell(x),
$$
where $B$ is the update batch and $\ell$ is the loss function. For a token $i$, $\nabla_{E_i}\ell(x)$ is null if the sequence $x$ does not contain the token $i$, leading to the gradient for many embedding vectors to be zero. 
When we tokenize a dataset, each token $i \in [m]$ has some probability $\alpha_i$ of appearing in an input sequence. As a result, when the batch size is sufficiently large,\footnote{Here, we just need that the batch size satisfies $|B| \gg T$, which is generally the case for LLM pretraining when batch size is usually of the order of millions, and vocab size is of the order of $10^4$ to $10^5$.} we can rewrite the gradient $G_i$ as 
$$
G_i \approx \sum_{ j=1 }^m \alpha_j \nabla_{E_i}\ell(j).
$$
This is a training setup where the input has a discrete distribution (one-hot vectors). We further refine this observation below.

\paragraph{Simple Model with LLM-like Embedding Layer.} The the embedding layer and the projection layer interact directly via the residual stream. This interaction leads to important changes in the training dynamics as vocabulary size increases. To model this interaction, we consider a simple linear network consisting only of the embedding layer and a projection layer, and is given by
\begin{equation}
\begin{cases}
Y_{emb}(x) = x^\top E \in \reals^{1\times d},\\
f(x)= Y_{emb}(x) W  = x^\top E W,
\end{cases}
\end{equation}
 where $x_i \in \reals^{m}$  denotes the input vector (a one-hot vector), $E \in \reals^{m \times d}$ is the embedding matrix, and $W\in \reals^{d\times m}$ is the projection matrix.\footnote{In this simple model, we do not have an attention mechanism. Therefore each token is processed independently, and there is no direct interaction between tokens in a sequence.} 
 
\paragraph{Large-Batch Training.} As is common in practice, we consider large batch size $N$. We assume that the model is trained by minimizing the square loss $\ell(z,z') = (2m)^{-1}\|z-z'\|^2$ for $z,z' \in \reals^m$.\footnote{The square loss can be replaced with any sufficiently regular loss function.} The dataset $\data$ consists of input-target pairs $(x, z)$. Averaging over the batch, the gradients are given by 
\begin{align*}
    dE =  \left(\frac{1}{N} \sum_{i=1}^N x_i \chi(x_i) \right)  W^\top , \quad dW = E^\top \left(\frac{1}{N} \sum_{i=1}^N x_i\, \chi(x_i) \right),
\end{align*}
where $\chi(x_i) = \nabla_{f(x_i)}\ell(f(x_i), z_i) = m^{-1} (x_i E W - z_i) \in \reals^{1\times m}$.
Let the vectors $(u_i)_{1\leq i \leq m} \in \reals^m$ denote the one-hot vectors in dimension $m$, and scalars $\alpha_i \in [0,1]$ satisfying $\sum_{i=1}^m \alpha_i = 1$ such that the distribution of $x$ follows
$$
\mu(x) = \sum_{i=1}^m \alpha_i \delta(u_i),
$$
where $\delta(u)$ refers to the Dirac mass at $u$. $\alpha_i$ is the probability of seeing token $i$ in the dataset.\footnote{In LLMs, the inputs $x_i$ come from a discrete $m$-dimensional distribution where the support vectors are just the one-hot vectors.} With such discrete input distribution, we can write the gradients as follows
\begin{align*}
    dE =  \left( \sum_{i=1}^m \hat{\alpha}_i u_i \chi(u_i) \right)  W^\top, \quad dW = E^\top \left(\sum_{i=1}^m \hat{\alpha}_i u_i \chi(u_i)\right),
\end{align*}
where $\hat{\alpha}_i = \frac{\#\{ i \in [N]: x_i = u_i\} }{N}$ are the empirical frequencies.\footnote{Here, we use frequency and probability interchangeably to refer to probabilities $\alpha_i$.}

With large batch size ($N\to \infty$), the training dynamics can be approximated by their infinite-batch counterpart
\begin{align*}
    dE &=  \left( \sum_{i=1}^m \alpha_i u_i \chi(u_i) \right)  W^\top = D_\alpha (EW - Z) W^\top\\
    dW &= E^\top \left(\sum_{i=1}^m \alpha_i u_i \chi(u_i)\right) = E^\top D_\alpha (EW - Z),
\end{align*}
where $\alpha_i$ are the \emph{real} frequencies from the data distribution, $D_\alpha= \textrm{Diag}((\alpha)_{1\leq i \leq m})$ is the diagonal matrix consisting of the frequencies $\alpha_i$ in its diagonal, and $Z = (z_1^\top, z_2^\top, \dots, z_m^\top)^\top \in \reals^{m\times m}$ are the concatenated targets. This infinite-batch limit captures the impact of the frequencies $\alpha_i$ on the updates of the embedding vectors. Specifically, for token $i$, we have $dE_i = \alpha_i \chi(u_i) W^\top$, which is proportional to the frequency $\alpha_i$. This is expected as embedding vectors corresponding to less frequent tokens receive less significant updates.

For the remainder of this section, we use the infinite-batch limit approximation to derive scaling limits of feature learning. We analyze training dynamics after taking one step with SignSGD -- a simplified version of Adam -- and show that as vocabulary size increases, the dynamics interpolate between the $\mu$P regime and the Large Vocab (LV) regime.

\subsection{Analysis with Adam-like Optimizers}
In practice, when training a large model, we often use Adam \citep{kingma2017adammethodstochasticoptimization} or its variants. Unlike gradient descent, Adam-like optimizers process the raw gradient to incorporate momentum and normalization. For scaling dynamics, normalization is particularly significant because it fundamentally changes the magnitude of gradients, which becomes $\Theta(1)$ across virtually all scalable dimensions. In order to have tractable analysis with such optimizers, we consider SignSGD, a momentum-less version of Adam that replaces each coordinate in the gradient with its sign element-wise, and we have
\begin{align*}
E_{t+1} &= E_t - \eta_E \, \Sign(D_\alpha (E_t W_t - Z) W_t^\top),\\
W_{t+1} &= W_t - \eta_W \Sign(E_t^\top D_\alpha (E_t W_t - Z))
\end{align*}
where $\Sign(v) = (\textrm{sign}(v_i))_{v_i \in v}$ is of the same size and dimension as $v$.\footnote{Specifically, given a matrix/vector $z = (z_i)_{1\leq i \leq p} $, $\Sign(z) = (\Sign(z_i))_{1\leq i\leq p}$, where $\Sign(z_i) = 1$ if $z_i\geq 0$ and $0$ otherwise.}

After one step, the output for a given token $i \in [m]$ changes as follows
\begin{align*}
E_{i,1} W_1 =  E_{i,0} W_0 &- \underbrace{\eta_W E_{i,0} \, \Sign(E_0^\top D_\alpha (E_0 W_0 - Z))}_{\delta_W^i}
-\underbrace{\eta_E \Sign((E_{i,0} W_0 - Z)W_0^\top)W_0}_{\delta_E^i}\\
& + \underbrace{\eta_E \eta_W \Sign( (E_{i,0} W_0 - z_i) W_0^\top) \Sign(E_0^\top D_\alpha (U^\top E_0 W_0 - Z)}_{\delta_{W,E}^i}.
\end{align*}
The terms $\delta_W^i$ and $\delta_E^i$ represent \emph{feature learning} components corresponding to updates in $W$ and $E$ respectively. Specifically, $\delta_W^i$  corresponds to updating $W$ while keeping $E$ fixed, and vice-versa for $\delta_E^i$. They capture the respective contributions of $W$ and $E$ to feature learning after a single optimization step. The term $\delta_{W,E}^i$ captures the multiplicative effect of both components. $\mu$P-style theory defines efficient feature learning as having both $\delta_W^i = \Theta_d(1)$ and $\delta_E^i=\Theta_d(1)$, element-wise. The intuition is straightforward: we want to avoid instability in training dynamics (implying components should be at most $\mathcal{O}_d(1)$), while ensuring both parameters contribute effectively to feature learning (avoiding cases where, for instance, $\delta_E^i = o_d(1)$ and $\delta_W^i = \Theta_d(1)$). It can be shown that having both $\delta_E^i=\Theta_d(1)$ and $\delta_W^i = \Theta_d(1)$ implies that the multiplicative feature learning component also satisfies $\delta_{W,E}^i=\Theta_d(1)$.

One limitation of $\mu$P is that vocabulary size $m$ is considered fixed and only $d$ goes to infinity. Therefore, it is unclear how $m$ affects feature learning components $\delta_E^i$ and $\delta_W^i$ for different tokens $i$. In the next result, we address this question and characterize the asymptotic behavior of $\delta_W^i$ and $\delta_E^i$ when both width $d$ and vocabulary size $m$ are large.

\begin{thm}[Informal]\label{thm:signsgd_asymptotics}
Consider the following setup and notation:
\begin{itemize}
    \item Initialization: $W \sim \normal(0,\sigma_W^2)$ and  $ E\sim \normal(0, \sigma_E^2)$ (iid).
    \item For $i\in [m]$, the quantities $\bar{\delta}_E^i = \left(m^{-1}\E \|\delta_E^i\|^2\right)^{1/2}$ and $\bar{\delta}_W^i = \left(m^{-1}\E \|\delta_W^i\|^2\right)^{1/2}$ denote the average norm of $\delta_E^i$ and $\delta_W^i$ respectively,\footnote{The average norm $\bar{\delta}_E^i = \left(m^{-1}\E \|\delta_E^i\|^2\right)^{1/2}$ is a good indicator of the magnitude of the coordinates of $\delta_E^i$ since these coordinates have the same distribution.} where the expectation is w.r.t initialization weights.
    \item  \(\displaystyle
  \bar\alpha^{2}\;:=\;\frac1m\sum_{k=1}^{m}\alpha_k^{2}\) denotes the average squared-frequencies.
\end{itemize}
Then, for all $i\in [m]$, the following holds
\begin{itemize}
    \item Embedding: $\bar{\delta}_E^i = \Theta_{m,d}\left(\eta_E \sigma_W \sqrt{d + \frac{2d(d-1)}{\pi m})}\right)$.
    \item Projection: $\bar{\delta}_W^i = \Theta_{m,d}\left(\eta_W \sigma_E \sqrt{d + \frac{\alpha_i^2}{\bar\alpha^2}\frac{2d(d-1)}{\pi m})}\right)$.
\end{itemize}
\end{thm}

\cref{thm:signsgd_asymptotics} characterizes how feature learning components corresponding $\delta_E$ and $\delta_W$ as both with $d$ and vocabulary size $m$ grow. The formal statement and proof are provided in \cref{sec:proofs}.  Note that token frequencies $\alpha_i$ naturally depend on $m$. 

If we fix $m$, as in $\mu$P, we obtain exactly the $\mu$P scaling rules. Specifically, with $m$ fixed, the asymptotics in \cref{thm:signsgd_asymptotics} become $\bar{\delta}_E^i = \Theta_d(\eta_E \sigma_W d)$ and $\bar{\delta}_W^i = \Theta_d(\eta_W \sigma_E d)$. With $\sigma_W = \Theta_d(d^{-1/2})$, $\sigma_E = \Theta_d(1)$, $\eta_W = \Theta_d(d^{-1})$, and $\eta_E = \Theta_d(1)$ ($\mu$P), these asymptotics satisfy $\bar{\delta}_W^i = \Theta_d(1)$ and $\bar{\delta}_E^i= \Theta_d(1)$. 

However, in practice $m$ is typically large and often significantly larger than $d$. Therefore, considering a fixed $m$ is an unrealistic assumption. If we let vocabulary size $m$ grow, we obtain different scaling behaviors. One complication is the dependence of the frequencies $\alpha_i$ on vocabulary size $m$. In order to control how these frequencies change with $m$, we assume that token distribution follows the Zipf-Mandelbrot law. 

\begin{wrapfigure}{r}{0.28\textwidth}
  \begin{center}
    \includegraphics[width=0.28\textwidth]{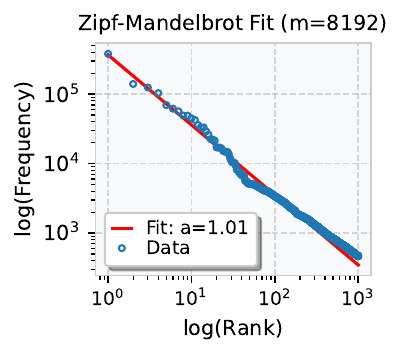}\\
    \vspace{-0.2cm}
    \includegraphics[width=0.28\textwidth]{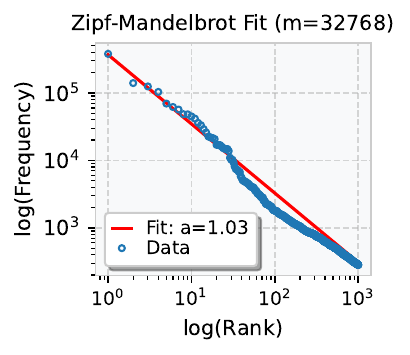}
  \end{center}
  \caption{\small{Verification of the Zipf-Mandelbrot Law assumption for $m \in \{8192, 32768\}$. In both cases, we train a BPE tokenizer on Wikitext2, tokenize the dataset, and show the frequencies. }}\label{fig:zipf_law}
  \vspace{-3em}
\end{wrapfigure}

\begin{assumption}[Zipf-Mandelbrot Law \citep{Zipf1932}]\label{assump:zipflaw}
Without loss of generality assume that tokens are ranked by their frequencies. For vocab size $m$, the frequencies are given by $\alpha_i \propto i^{-a_m}$ for some $a_m>0$.
\end{assumption}
The Zipf-Mandelbrot Law is a power law frequently observed in ranked data. It is well-studied in probability and statistics, and the coefficient $a_m$ is usually close to $1$. \cref{fig:zipf_law} verifies this assumption by showing token frequencies for $m \in \{8192, 32768\}$, where in both cases we trained  a BPE tokenizer on Wikitext2 \citep{merity2016wikitext} and tokenized the same dataset to obtain frequencies. The results demonstrate that this power law is an excellent approximation of token distribution. Under this assumption, the next result shows that the quantity $\bar{\alpha}^2$ has $\Theta(m)$ asymptotic growth with vocabulary size $m$ which leads to scaling rules that are different from $\mu$P.
\begin{lemma}[Average Frequency vanishes with Vocabulary Size]\label{lemma:alpha_asymptotics}
Under \cref{assump:zipflaw}, assume that there exist constants $\munderbar{a}, \bar{a} > 1/2$ such that for all $m \geq 1, a_m \in [\munderbar{a}, \bar{a}]$. Then, we have 
$$
\bar\alpha = m^{-1} \sum_{i=1}^m i^{-2a_m} = \Theta_m(m^{-1}).
$$
\end{lemma}
\begin{proof}
The proof follows from noting that $\sum_{i=1}^m i^{-2 \munderbar{a}} = \Theta_m(1)$ and $\sum_{i=1}^m i^{-2 \bar{a}} = \Theta_m(1)$ (standard series convergence result).
\end{proof}
\cref{lemma:alpha_asymptotics} shows that under the power law assumption on token frequencies (with realistic values for $a$), the average squared frequency vanishes as $1/m$ with vocabulary size $m$. This is expected because with such power laws, most of the distribution weight is concentrated among the frequent tokens.

Using the result of \cref{lemma:alpha_asymptotics}, we show in the theorem that the asymptotic behavior of $\bar{\delta}_W^i$ and $\bar{\delta}_E^i$ differs from the $\mu$P asymptotics when both $d$ and $m$ are large, leading to different scaling rules.
\begin{thm}[Large Vocabulary Regime, (Informal)]\label{thm:large_vocab_regime}
Let $k$ be a fixed integer, independent of $m$. Suppose that \cref{assump:zipflaw} is satisfied, and assume there exist constants $\munderbar{a}, \bar{a} > 1/2$ such that for all $m \geq 1, a_m \in [\munderbar{a}, \bar{a}]$. Then, using the same notation of \cref{thm:signsgd_asymptotics}, for any $i \leq k$
\begin{itemize}
    \item Embedding: $\bar{\delta}_E^i = \Theta_{m,d}\left(\eta_E \sigma_W \sqrt{d + \frac{2d(d-1)}{\pi m})}\right).$
    \item Projection: $\bar{\delta}_W^i = \Theta_{m,d}\left(\eta_W \sigma_E \sqrt{d + \frac{2}{\pi }d(d-1))}\right).$
\end{itemize}
\end{thm}
One of the assumptions in \cref{thm:large_vocab_regime} is that $k$ is a fixed integer and the results hold for $i \leq k$. Intuitively, this should be interpreted as the result holds for high frequency tokens. For these tokens, the $d$ term in the asymptotic formula of $\bar{\delta}_W^i$ is amplified by the $\Theta(m)$ from \cref{lemma:alpha_asymptotics}, which brings back the asymptotics of $\bar{\delta}_W^i$ to the $\mu$P regime where $\bar{\delta}_W^i = \Theta(\eta_W \sigma_E d)$, while preserving the $1/m$ term in the asymptotics of $\bar{\delta}_E^i$. Consequently, we should expect the impact of vocabulary size to be more pronounced on the embedding layer learning rate. This aligns with empirical observations, where $\mu$P scaling rule for the embedding layer does not achieve HP transfer as shown in \cref{fig:no_hp_transfer_with_mup}. Details about the formal statement of \cref{thm:large_vocab_regime} and its proof are provided in \cref{sec:proofs}.

\paragraph{Interpolation between regimes.} As $m$ grows, the asymptotics interpolate between the $\mu$P regime where $d$ term dominates in both $\bar{\delta}_E^i$ and $\bar{\delta}_W^i$, and the Large Vocabulary (LV) regime where $\sqrt{d}$ term dominates in $\bar{\delta}_E^i$ and $d$ term dominates in $\bar{\delta}_W^i$. In the latter, $\bar{\delta}_E^i$ becomes asymptotically $\Theta(\eta_E \sigma_W \sqrt{d})$, which differs from $\mu$P regime. In the LV regime, $\mu$P is no longer optimal. Empirically, we found that the best initialization in the LV regime is $\sigma_W = \sigma_E = d^{-1/2}$ which is the initialization used in standard parametrization (SP), which differs from $\mu$P's $\sigma_W=d^{-1}$ and $\sigma_E = 1$. This leads to $\eta_E/\eta_W = \Theta(\frac{\sigma_E}{\sigma_W} \sqrt{d}) = \Theta(\sqrt{d})$ unlike the $\Theta(d)$ ratio predicted by $\mu$P. Note that SP initialization was also found to lead to HP transfer in \cite{everett2024scalingexponentsparameterizationsoptimizers}, provided that the learning rates exponents are adapted accordingly.

\paragraph{What about low frequency tokens?} The result of \cref{thm:large_vocab_regime} holds for the most frequent tokens $i$. If we consider less frequent tokens $i$, the asymptotic behavior of $\bar{\delta}_W^i$ changes significantly. For instance, if we consider a token $i_m$ such that $\alpha_{i_m}^2 = \Theta(m^{-\beta})$, then we have $\bar{\delta}_W^{i_m} = \Theta_{m,d}\left(\eta_W \sigma_E \sqrt{d + m^{-\beta + 1} \times \frac{2 d(d-1))}{\pi m}}\right).$ If $\beta>1$, then in this case the $d^2$ term is actually amplified in contrast to the behavior of the most frequent tokens. However, it should be expected that the most frequent tokens are more important during training and therefore optimal parametrizations are generally tuned so that feature learning for such tokens is optimal.

An important conclusion from this analysis is that $\mu$P is suboptimal for LLM training, especially in how embedding learning rate is parametrized. Large vocabulary size $m$ reduces the signal strength in $\delta_E^i$ from $d$ to $\sqrt{d}$, leading to changes in optimal scaling rules.  In the next section, we show that using a ratio of $\eta_E/\eta_W = \sqrt{d}$ yields better performance when the vocabulary size is large ($m \gg d$), a condition generally satisfied in practice.

\paragraph{Extension to Transformer Architectures.} The theoretical analysis presented in this work uses a simplified model consisting only of embedding and projection layers. Modern LLMs have hidden layers between embedding and projection layers. Despite this complexity, we expect our core findings to remain relevant to full transformer architectures due to a critical architectural feature: the residual stream, which creates a direct connection between the embedding and projection layers. While a comprehensive theoretical analysis of the full transformer architecture is not possible using the mathematical framework introduced in this paper, the fundamental dynamics we have identified likely hold. The direct pathway created by the residual connections preserves the essential relationship between the embedding and projection layers that drives the scaling behaviors we have observed.

\section{Experiments}\label{sec:exps}
Form the theoretical analysis in \cref{sec:theory}, we found that in the LV regime, feature learning dynamics are different from those associated with $\mu$P, and concluded that when the vocabulary size $m$ is large enough (i.e. $m \gg d$), a good rule of thumb is to set $\eta_E$ and $\eta_W$ such that $\eta_E/\eta_W =\sqrt{d}$. We call this the \emph{$\sqrt{d}$-rule}. In the more realistic scenario of transformers architecture, $\eta_W$ is the learning rate used for hidden layers and projection layer. More generally, we define the \emph{Large Vocabulary Parametrization (LVP)} which uses the same initialization as SP, the same learning rate of $\mu$P for hidden and output projection, but incorporates the $\sqrt{d}$-rule in the embedding learning rate. \cref{tab:params} summarizes these different parametrizations.

\begin{table}[ht]
    \centering
    \caption{Different parametrization: SP (Standard Parametrization), $\mu$P (Maximal Update Parametrization), and LVP (Large Vocabulary Parametrization). We show how each parametrization sets the initialization variance and learning rate (Adam).}
    \begin{tabular}{c|c|c|c|c}
    \Xhline{2\arrayrulewidth}
       Param & HP & Embedding  &  Output & Hidden \\\hline
      SP & Init. Var. & $1$ & $d^{-1}$ & $d^{-1}$ \\
       \cline{2-5}
       & LR (Adam)  & $\eta$ & $\eta$ & $\eta$\\\hline
      $\mu$P & Init. Var. & $1$ & $d^{-2}$ & $d^{-1}$ \\
       \cline{2-5}
       & LR (Adam)  & $\eta$ & $\eta d^{-1}$ & $\eta d^{-1}$\\\hline
      LVP (ours) & Init. Var. & $d^{-1}$ & $d^{-1}$ & $d^{-1}$ \\
       \cline{2-5}
       & LR (Adam)  & $\eta d^{-1/2}$ & $\eta d^{-1}$ & $\eta d^{-1}$
       \\\Xhline{2\arrayrulewidth}
    \end{tabular}
    \label{tab:params}
\end{table}

Our goal in this section is to assess the effectiveness of this rule for both hyperparameter transfer and effective feature learning. We conducted two experiments:

\begin{enumerate}
    \item We pretrained a series of small language models consisting of only two hidden layers in addition to embedding and projection layers, where we scale both width $d$ and vocabulary size $m$. We found that the $\sqrt{d}$-rule provides near-optimal performance in this case, leading to better transfer of the embedding learning rate.

    \item  To evaluate the benefit of the $\sqrt{d}$-rule in a production-scale scenario, we pretrained a 1B model from scratch and measure perplexity for different ratios $\eta_E/\eta_W$. Our results show that setting $\eta_E/\eta_W =  \sqrt{d}$ yields improved training loss and perplexity.
\end{enumerate}

\begin{wrapfigure}{r}{0.28\textwidth}
\vspace{-1em}
  \begin{center}
    \includegraphics[width=0.28\textwidth]{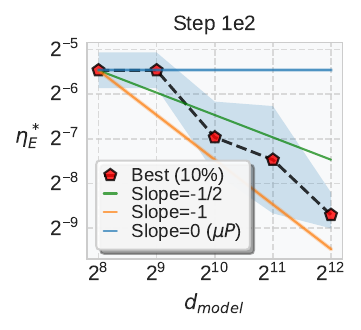}\\
    \vspace{-0.2cm}
    \includegraphics[width=0.28\textwidth]{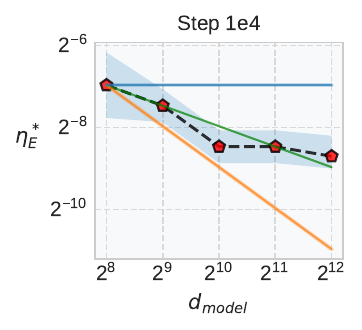}
  \end{center}
  \caption{\small{Optimal $\eta_E$ vs $d_{model}$. The shaded curves is calculated as the geometric mean of all learning rates that give a train loss within $20\%$ of the best train loss. The average is calculated across learning rate grid and random seeds.}}\label{fig:emb_lr_vocab_scaling}
  \vspace{-2em}
\end{wrapfigure}
\subsection{Increasing Vocabulary Size}
% \begin{figure}[t]
%     \centering
%     \includegraphics[width=0.3\linewidth]{figures/optimal_emb_lr_step1e2_vocab_scaling.pdf}
%     \includegraphics[width=0.3\linewidth]{figures/optimal_emb_lr_step1e4_vocab_scaling.pdf}
%     \caption{Caption}
%     \label{fig:emb_lr_vocab_scaling}
% \end{figure}
We trained a small decoder-only transformer with the following configuration: an embedding layer, two hidden layers, and a projection layer. We used Wikitext2 as our training dataset, with a maximum sequence length of $256$ and fixed positional encoding added to the embedding layer. Experiments were conducted for various configurations $(d,m) \in \{(2^k, 2^{k+3}), k=8, \dots, 12\}$. For each vocabulary size $m$, we trained a dedicated  BPE tokenizer on Wikitext2. In the experiments previously reported in \cref{fig:no_hp_transfer_with_mup}, we observed that $\mu$P scaling for hidden and projection layers LRs remains robust as we scale $d$. Therefore, LVP (\cref{tab:params}) uses the same scaling rules as $\mu$P for hidden and output layers. The learning rate for these layers is $\eta_W = \eta d^{-1}$ where $\eta \approx 0.2$ is obtained by grid tuning for the model with $d=2^8$. Similarly, we found that setting $\sigma_W = \sigma_E = d^{-1/2}$ generally gives better performance than $\mu$P's init, so we use that in LVP. We do a grid search over $\eta_E$ and run three seeds per experiment. Each run consists of $1e^4$ Adam steps. See \cref{app:emp} for more details about experimental setup.

Our choice of $(d,m)$ configurations ensures that the vocabulary size scales linearly with $d$. This approach makes the term $d(d-1)/m$ in \cref{thm:large_vocab_regime} comparable to $d$, the first term in that asymptotic formula.  

\cref{fig:emb_lr_vocab_scaling} summarizes our experimental findings. The results confirm that optimal embedding LR decreases with $d$ in contrast to the $\Theta_d(1)$ predictions of $\mu$P. Setting $\eta_E = \Theta(d^{-1/2})$ (or equivalently, $\eta_E/\eta_W = \Theta(d^{1/2})$) yields near-optimal training loss ($t=1e4$).
However, while the $\sqrt{d}$-rule is generally better than $\mu$P or SP, the optimal embedding learning rate exhibits significant variance around the $\sqrt{d}$ line. This might indicate more fundamental limitations in the parametrization. See \cref{sec:discussion} for more details.

\subsection{Production-Scale LLM Pre-training}

We conduct experiments on LLM pre-training using a 1B parameter dense Transformer model. This is a 24 layers model with width $d=2048$.

\smallsection{Training Corpus} To better approximate production-level LLM pre-training, we trained our model on 1.75T tokens as a Causal Language Model. This training dataset is identical to the one used to train Phi-3 dense models~\citep{abdin2024phi}.

\smallsection{Model Architecture} We implemented a modern Transformer architecture with RoPE embeddings, sliding window multi-query attention, and SwiGLU activation. The detailed model configuration is available in Table~\ref{tab:model_params}.

\smallsection{Pre-training Performance} We trained two models for 1.75T tokens. The baseline model used the conventional LLM training configuration, which applies the same learning rate to both the embedding layer and the remaining parameters. The experimental model implemented our proposed scaling rule, which sets the embedding layer learning rate to be $\sqrt{d}$ times larger than for the remaining parameters.

Figure~\ref{fig:training_ppl} visualizes the training perplexity curves for both models, showing that our proposed method leads to consistently faster convergence throughout training and achieves a better final training loss. To further verify the effectiveness of our proposed scaling rule, we experimented with additional learning rate ratios on the 1B model. As shown in Figure~\ref{fig:training_ppl}, \emph{our proposed $\sqrt{d}$ scaling rule leads to near-optimal performance, while further increasing the learning rate yields only marginal benefits}.

\smallsection{Evaluation on Wikitext} We measured the word-level perplexity on the Wikitext test corpus~\cite{merity2016wikitext}. As visualized in Figure~\ref{fig:test_ppl}, our proposed model consistently outperforms the conventional baseline.

\begin{figure}
\begin{minipage}{\textwidth}
\begin{minipage}[b]{0.61\textwidth}
\centering
\includegraphics[width=0.95\linewidth]{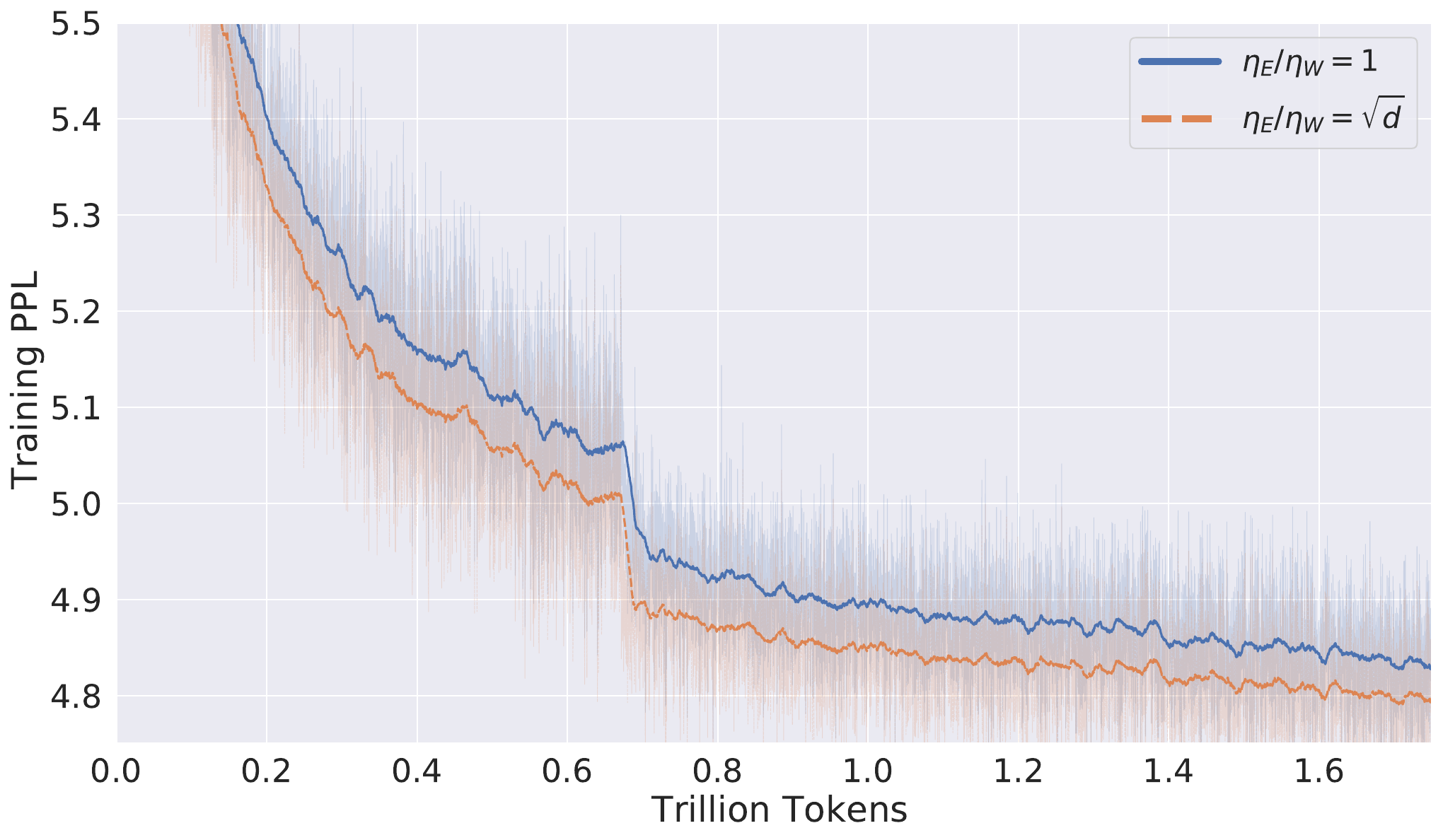}
\subcaption{slope=0 and slope=-1/2 on 1.75T tokens.}
\label{fig:training_ppl}
\end{minipage}
\hfill\,
\begin{minipage}[b]{0.37\textwidth}
\centering
\includegraphics[width=0.95\linewidth]{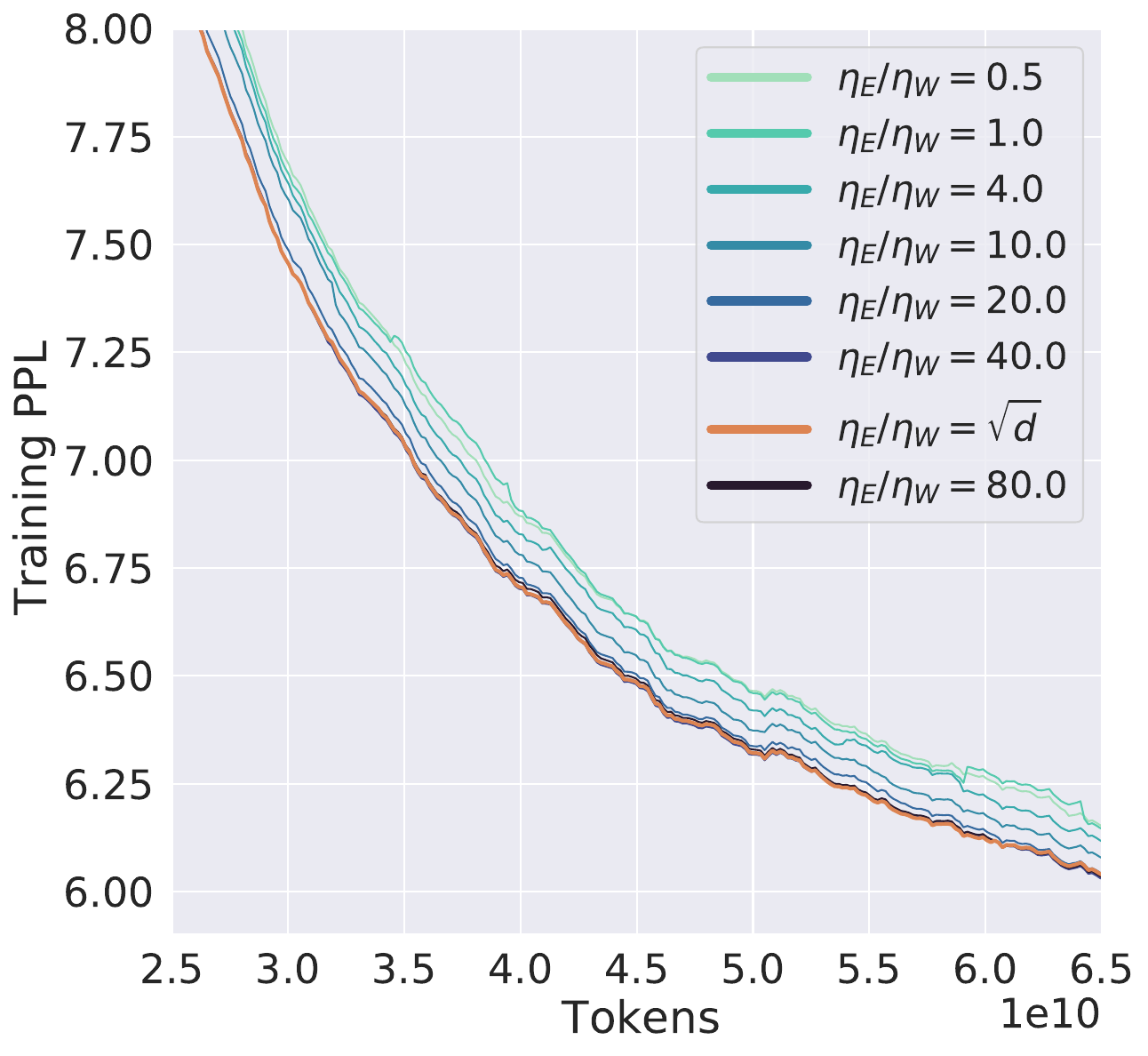}
\subcaption{different $\eta_E/\eta_W$ on 65B tokens.}
\label{fig:training_ppl_different_ratios}
\end{minipage}
\end{minipage}
\caption{LLM Pre-training Perplexity of 1B Transformer. Note $\sqrt{d}\approx45.3$ here.}
\end{figure}

% \begin{figure}[t]
%     \centering
%     \includegraphics[width=0.8\linewidth]{figures/ppl_training.pdf}
%     \caption{LLM Pre-training Perplexity of 1B Transformer}
%     \label{fig:training_ppl}
% \end{figure}

\begin{figure}[t]
    \centering
    \includegraphics[width=0.7\linewidth]{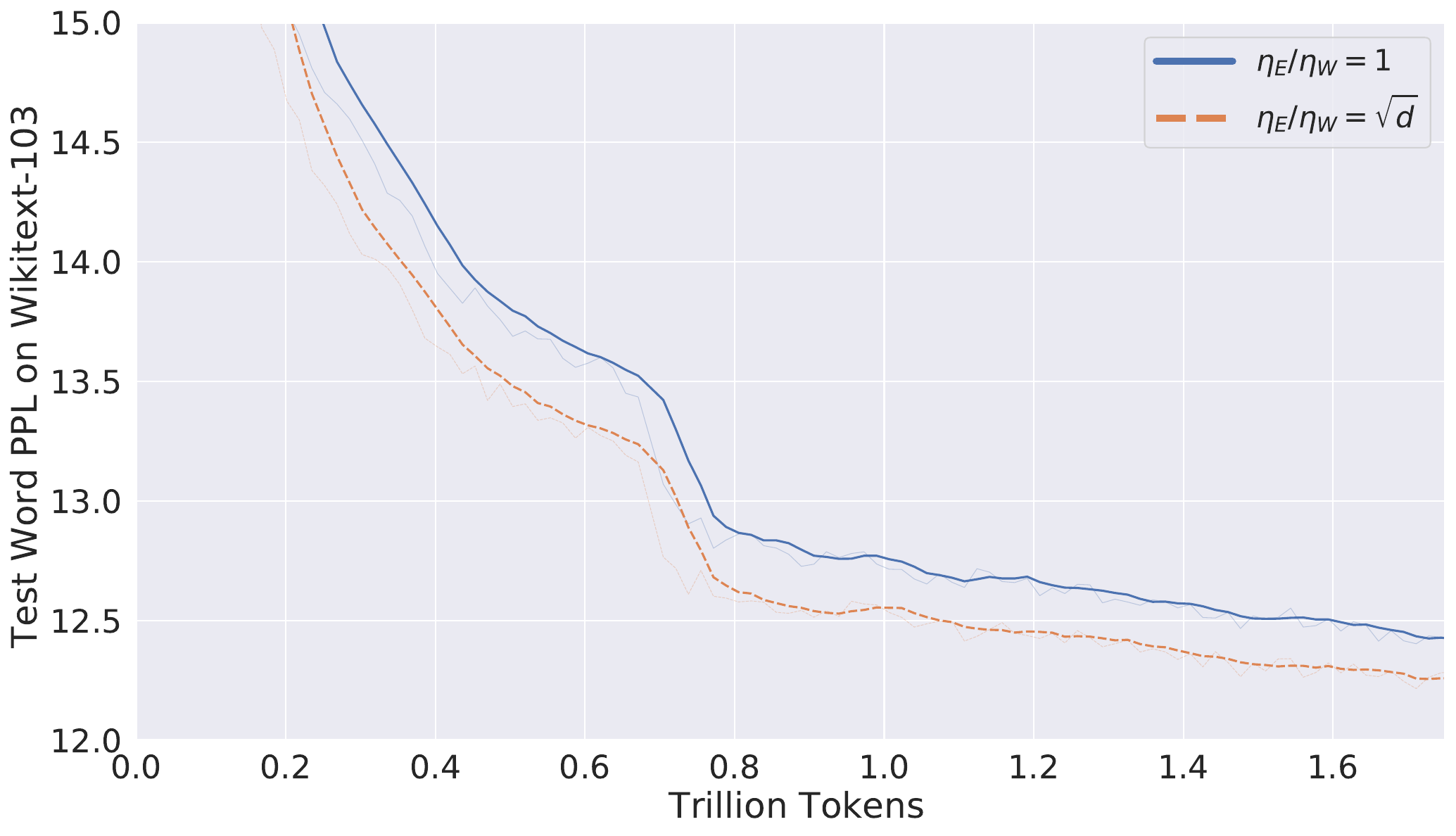}
    \caption{Perplexity of 1B Transformer on the Wikitext Test Set}
    \label{fig:test_ppl}
\end{figure}
\section{Discussion and Limitations}\label{sec:discussion}

In this work, we showed that large vocabulary size leads to different scaling rules than those predicted by $\mu$P. In this regime, the optimal ratio of embedding LR to hidden LR appears to roughly of order $\sqrt{d}$, which is different from the order $d$ ratio suggested by $\mu$P, and the ratio of 1 used in SP. 

While our suggested $\sqrt{d}$-rule yields better performance even for production-scale LLMs (compared to standard practice), it is important to understand that this rule is more about efficient feature learning than hyperparameter transfer. Specifically,  hyperparameter transfer occurs when the parametrization is optimal, in the sense that it leads to the most effective feature learning limit. It is unclear whether the limit with this scaling is optimal, and further theoretical and empirical studies are needed to investigate this. In our discussion in \cref{sec:theory}, we mentioned that optimal scaling rules for the learning rate are sensitive to token frequency, and therefore, it should not be unexpected that optimal parametrization should incorporate information about token frequency. 

Our theoretical results were derived in a simplified setting where we consider a simple model trained with SignSGD with 1 step . We expect the results to hold even for multiple steps with more practical optimizers such as Adam since the main ingredient in our analysis is the normalization process from such optimizers. The number of steps $t$ is important and we expect the results to hold when $t$ is fixed and $d, m$ go to infinity. The more general case where $t$ is free is more challenging.  

Finally, the interaction between embedding and projection layers is present in modern LLMs via the residual stream, which indicates that the large vocabulary regime is relevant for LLMs as well. Our empirical results further confirm this.

\section{Acknowledgment}
SH is supported by the NSF and the Simons Foundation for the Collaboration on the Theoretical Foundations of Deep Learning through awards DMS-2031883 and 81463.\\
\newpage

\bibliography{references}

\begin{thebibliography}{27}
\providecommand{\natexlab}[1]{#1}
\providecommand{\url}[1]{\texttt{#1}}
\expandafter\ifx\csname urlstyle\endcsname\relax
  \providecommand{\doi}[1]{doi: #1}\else
  \providecommand{\doi}{doi: \begingroup \urlstyle{rm}\Url}\fi

\bibitem[Blake et~al.(2025)Blake, Eichenberg, Dean, Balles, Prince, Deiseroth,
  Cruz-Salinas, Luschi, Weinbach, and
  Orr]{blake2025umupunitscaledmaximalupdate}
Charlie Blake, Constantin Eichenberg, Josef Dean, Lukas Balles, Luke~Y. Prince,
  Björn Deiseroth, Andres~Felipe Cruz-Salinas, Carlo Luschi, Samuel Weinbach,
  and Douglas Orr.
\newblock u-$\mu$p: The unit-scaled maximal update parametrization, 2025.
\newblock URL \url{https://arxiv.org/abs/2407.17465}.

\bibitem[Bordelon et~al.(2023)Bordelon, Noci, Li, Hanin, and
  Pehlevan]{bordelon2023depthwisehyperparametertransferresidual}
Blake Bordelon, Lorenzo Noci, Mufan~Bill Li, Boris Hanin, and Cengiz Pehlevan.
\newblock Depthwise hyperparameter transfer in residual networks: Dynamics and
  scaling limit, 2023.
\newblock URL \url{https://arxiv.org/abs/2309.16620}.

\bibitem[Chizat et~al.(2020)Chizat, Oyallon, and
  Bach]{chizat2020lazytrainingdifferentiableprogramming}
Lenaic Chizat, Edouard Oyallon, and Francis Bach.
\newblock On lazy training in differentiable programming, 2020.
\newblock URL \url{https://arxiv.org/abs/1812.07956}.

\bibitem[Chizat et~al.(2024)Chizat, Colombo, Fernández-Real, and
  Figalli]{chizat2024infinitewidth}
Lénaïc Chizat, Maria Colombo, Xavier Fernández-Real, and Alessio Figalli.
\newblock Infinite-width limit of deep linear neural networks.
\newblock \emph{Communications on Pure and Applied Mathematics}, 77\penalty0
  (10):\penalty0 3958--4007, 2024.
\newblock \doi{https://doi.org/10.1002/cpa.22200}.
\newblock URL \url{https://onlinelibrary.wiley.com/doi/abs/10.1002/cpa.22200}.

\bibitem[Everett et~al.(2024)Everett, Xiao, Wortsman, Alemi, Novak, Liu, Gur,
  Sohl-Dickstein, Kaelbling, Lee, and
  Pennington]{everett2024scalingexponentsparameterizationsoptimizers}
Katie Everett, Lechao Xiao, Mitchell Wortsman, Alexander~A. Alemi, Roman Novak,
  Peter~J. Liu, Izzeddin Gur, Jascha Sohl-Dickstein, Leslie~Pack Kaelbling,
  Jaehoon Lee, and Jeffrey Pennington.
\newblock Scaling exponents across parameterizations and optimizers, 2024.
\newblock URL \url{https://arxiv.org/abs/2407.05872}.

\bibitem[Hayou(2022)]{hayou2022on}
Soufiane Hayou.
\newblock On the infinite-depth limit of finite-width neural networks.
\newblock \emph{Transactions on Machine Learning Research}, 2022.

\bibitem[Hayou et~al.(2019)Hayou, Doucet, and Rousseau]{hayou19activation}
Soufiane Hayou, Arnaud Doucet, and Judith Rousseau.
\newblock On the impact of the activation function on deep neural networks
  training.
\newblock In Kamalika Chaudhuri and Ruslan Salakhutdinov, editors,
  \emph{Proceedings of the 36th International Conference on Machine Learning},
  volume~97 of \emph{Proceedings of Machine Learning Research}, pages
  2672--2680. PMLR, 09--15 Jun 2019.
\newblock URL \url{https://proceedings.mlr.press/v97/hayou19a.html}.

\bibitem[Hayou et~al.(2021)Hayou, Clerico, He, Deligiannidis, Doucet, and
  Rousseau]{hayou2021stableresnet}
Soufiane Hayou, Eugenio Clerico, Bobby He, George Deligiannidis, Arnaud Doucet,
  and Judith Rousseau.
\newblock Stable resnet.
\newblock In Arindam Banerjee and Kenji Fukumizu, editors, \emph{Proceedings of
  The 24th International Conference on Artificial Intelligence and Statistics},
  volume 130 of \emph{Proceedings of Machine Learning Research}, pages
  1324--1332. PMLR, 13--15 Apr 2021.

\bibitem[He et~al.(2015)He, Zhang, Ren, and
  Sun]{he2015delvingdeeprectifierssurpassing}
Kaiming He, Xiangyu Zhang, Shaoqing Ren, and Jian Sun.
\newblock Delving deep into rectifiers: Surpassing human-level performance on
  imagenet classification, 2015.
\newblock URL \url{https://arxiv.org/abs/1502.01852}.

\bibitem[Jacot et~al.(2020)Jacot, Gabriel, and
  Hongler]{jacot2020neuraltangentkernelconvergence}
Arthur Jacot, Franck Gabriel, and Clément Hongler.
\newblock Neural tangent kernel: Convergence and generalization in neural
  networks, 2020.
\newblock URL \url{https://arxiv.org/abs/1806.07572}.

\bibitem[Jordan et~al.(2024)Jordan, Jin, Boza, Jiacheng, Cesista, Newhouse, and
  Bernstein]{jordan2024muon}
Keller Jordan, Yuchen Jin, Vlado Boza, You Jiacheng, Franz Cesista, Laker
  Newhouse, and Jeremy Bernstein.
\newblock Muon: An optimizer for hidden layers in neural networks, 2024.
\newblock URL \url{https://kellerjordan.github.io/posts/muon/}.

\bibitem[Kingma and Ba(2017)]{kingma2017adammethodstochasticoptimization}
Diederik~P. Kingma and Jimmy Ba.
\newblock Adam: A method for stochastic optimization, 2017.
\newblock URL \url{https://arxiv.org/abs/1412.6980}.

\bibitem[Lingle(2025)]{lingle2025empiricalstudymuplearning}
Lucas Lingle.
\newblock An empirical study of $\mu$p learning rate transfer, 2025.
\newblock URL \url{https://arxiv.org/abs/2404.05728}.

\bibitem[Llama-Team(2024)]{2024llama3herdmodels}
Llama-Team.
\newblock The llama 3 herd of models, 2024.
\newblock URL \url{https://arxiv.org/abs/2407.21783}.

\bibitem[Merity et~al.(2016)Merity, Xiong, Bradbury, and
  Socher]{merity2016wikitext}
Stephen Merity, Caiming Xiong, James Bradbury, and Richard Socher.
\newblock Pointer sentinel mixture models, 2016.

\bibitem[Neal(1996)]{Neal1996}
Radford~M. Neal.
\newblock Priors for infinite networks.
\newblock In \emph{Bayesian Learning for Neural Networks}, volume 118 of
  \emph{Lecture Notes in Statistics}, pages 29--53. Springer New York, 1996.
\newblock ISBN 978-0-387-94724-2.
\newblock \doi{10.1007/978-1-4612-0745-0_2}.

\bibitem[Poole et~al.(2016)Poole, Lahiri, Raghu, Sohl-Dickstein, and
  Ganguli]{poole2016exponential}
B.~Poole, S.~Lahiri, M.~Raghu, J.~Sohl-Dickstein, and S.~Ganguli.
\newblock Exponential expressivity in deep neural networks through transient
  chaos.
\newblock \emph{30th Conference on Neural Information Processing Systems},
  2016.

\bibitem[Schoenholz et~al.(2017)Schoenholz, Gilmer, Ganguli, and
  Sohl-Dickstein]{deepinfoprop2017}
S.S. Schoenholz, J.~Gilmer, S.~Ganguli, and J.~Sohl-Dickstein.
\newblock Deep information propagation.
\newblock In \emph{International Conference on Learning Representations}, 2017.

\bibitem[Tao et~al.(2024)Tao, Liu, Dou, Muennighoff, Wan, Luo, Lin, and
  Wong]{tao2024scalinglawsvocabularylarger}
Chaofan Tao, Qian Liu, Longxu Dou, Niklas Muennighoff, Zhongwei Wan, Ping Luo,
  Min Lin, and Ngai Wong.
\newblock Scaling laws with vocabulary: Larger models deserve larger
  vocabularies, 2024.
\newblock URL \url{https://arxiv.org/abs/2407.13623}.

\bibitem[Team(2023)]{falcon2023falconseriesopenlanguage}
Falcon Team.
\newblock The falcon series of open language models, 2023.
\newblock URL \url{https://arxiv.org/abs/2311.16867}.

\bibitem[Team(2025)]{gemmateam2025gemma3technicalreport}
Gemma Team.
\newblock Gemma 3 technical report, 2025.
\newblock URL \url{https://arxiv.org/abs/2503.19786}.

\bibitem[Team(2024)]{abdin2024phi}
Phi Team.
\newblock Phi-3 technical report: A highly capable language model locally on
  your phone.
\newblock \emph{arXiv:2404.14219}, 2024.

\bibitem[Yang(2019)]{yang2019scaling}
G.~Yang.
\newblock Scaling limits of wide neural networks with weight sharing: Gaussian
  process behavior, gradient independence, and neural tangent kernel
  derivation.
\newblock \emph{arXiv preprint arXiv:1902.04760}, 2019.

\bibitem[Yang et~al.(2022)Yang, Hu, Babuschkin, Sidor, Liu, Farhi, Ryder,
  Pachocki, Chen, and Gao]{yang2022tensor}
Greg Yang, Edward~J Hu, Igor Babuschkin, Szymon Sidor, Xiaodong Liu, David
  Farhi, Nick Ryder, Jakub Pachocki, Weizhu Chen, and Jianfeng Gao.
\newblock Tensor programs v: Tuning large neural networks via zero-shot
  hyperparameter transfer.
\newblock \emph{arXiv preprint arXiv:2203.03466}, 2022.

\bibitem[Yang et~al.(2023)Yang, Yu, Zhu, and
  Hayou]{yang2023tensorprogramsvifeature}
Greg Yang, Dingli Yu, Chen Zhu, and Soufiane Hayou.
\newblock Tensor programs vi: Feature learning in infinite-depth neural
  networks, 2023.
\newblock URL \url{https://arxiv.org/abs/2310.02244}.

\bibitem[Zhang et~al.(2025)Zhang, Morwani, Vyas, Wu, Zou, Ghai, Foster, and
  Kakade]{zhang2025doescriticalbatchsize}
Hanlin Zhang, Depen Morwani, Nikhil Vyas, Jingfeng Wu, Difan Zou, Udaya Ghai,
  Dean Foster, and Sham Kakade.
\newblock How does critical batch size scale in pre-training?, 2025.
\newblock URL \url{https://arxiv.org/abs/2410.21676}.

\bibitem[Zipf(1932)]{Zipf1932}
George~Kingsley Zipf.
\newblock \emph{Selected Studies of the Principle of Relative Frequency in
  Language}.
\newblock Harvard University Press, Cambridge, MA, 1932.

\end{thebibliography}

\newpage
\appendix

\section{Infinite-width analysis and $\mu$P}
\label{app:add_details_infinite_width}

As the width $d$ grows, model hyperparameters such as  initialization variance and learning should be adapted to avoid numerical instabilities and ensure efficient learning. For instance, the initialization variance should scale as $1/d$ to prevent arbitrarily large pre-activations as we increase model width $d$ (e.g. He init \cite{he2015delvingdeeprectifierssurpassing}). To derive such scaling rules, a principled approach consist of analyzing statistical properties of key quantities in the model (e.g. pre-activations) as $d$ grows and then adjust the initialization, the learning rate, and the architecture itself to achieve desirable properties in the limit $d \to \infty$ \cite{poole2016exponential, deepinfoprop2017,  hayou19activation, yang2019scaling, chizat2024infinitewidth}.

In this context, \citet{yang2022tensor} introduced the Maximal Update Parameterization (or $\mu$P), a set of scaling rules for the initialization scheme, the learning rate, and the network architecture that ensure stability and maximal feature learning in the infinite width limit. Stability is defined by $Y_l^i = \Theta(1)$ for all $l$ and $i$ where the asymptotic notation `$\Theta(.)$' is with respect to width $d$ (see next paragraph for a formal definition), and feature learning is defined by $\Delta Y_l = \Theta(1)$, where $\Delta$ refers to the feature update after taking a gradient step. $\mu$P guarantees that these two conditions are satisfied at any training step $t$. Roughly speaking, $\mu$P specifies that hidden weights should be initialized with $\Theta(d^{-1/2})$ random weights, and weight updates should be of order $\Theta(d^{-1})$. Input weights should be initialized $\Theta(1)$ and the weights update should be $\Theta(1)$ as well. While the output weights should be initialized $\Theta(d^{-1})$ and updated with $\Theta(d^{-1})$. These rules ensure both stability and feature learning in the infinite-width limit, in contrast to standard parameterization (exploding features if the learning rate is well tuned), and kernel parameterizations (e.g. Neural Tangent Kernel parameterization where $\Delta Y_l = \Theta(d^{-1/2})$, i.e. no feature learning in the limit). 
\section{Proofs}\label{sec:proofs}
In this section, we provide proofs for \cref{thm:signsgd_asymptotics} and \cref{thm:large_vocab_regime}. The proofs rely mainly on technical results from  \cref{sec:technical_results}.

\subsection{Proof of \cref{thm:signsgd_asymptotics}}
We use the same notation from \cref{sec:theory}. We require the following assumption on the target distribution.
\begin{assumption}[Random signal at initialization]\label{assumption:random_signal_at_initialization}
 Assume that $(EW - Z)_{ij} \sim \normal(0,1)$ are iid. 
\end{assumption}
\cref{assumption:random_signal_at_initialization} is a realistic approximation of practical setup in which the output at initialization $EW$ is randomly distributed so that the difference $EW-Z$ is random as well. Assuming that it is centered just means that the model is initialized in a way that makes outputs of the same range as targets.

\textbf{Theorem \ref{thm:signsgd_asymptotics}.} [Asymptotics in $d$ and $m$]\\
\emph{Consider the following notation:
\begin{itemize}
    \item Initialization: $W \sim \normal(0,\sigma_W^2)$ and  $ E\sim \normal(0, \sigma_E^2)$
    \item For $i\in [m]$, the quantities $\bar{\delta}_E^i = \left(m^{-1}\E \|\delta_E^i\|^2\right)^{1/2}$ and $\bar{\delta}_W^i = \left(m^{-1}\E \|\delta_W^i\|^2\right)^{1/2}$ denote the average norm of $\delta_E^i$ and $\delta_W^i$ respectively.\footnote{Note that the average norm $\bar{\delta}_E^i = \left(m^{-1}\E \|\delta_E^i\|^2\right)^{1/2}$ is a good indicator of the magnitude of the coordinates of $\delta_E^i$ since these coordinates have the same distribution.}
    \item  \(\displaystyle
  \bar\alpha^{2}\;:=\;\frac1m\sum_{k=1}^{m}\alpha_k^{2}\) denotes the average squared-frequencies.
\end{itemize}
Then, under \cref{assumption:random_signal_at_initialization}, for all $i\in [m]$, the following holds
\begin{itemize}
    \item $\bar{\delta}_E^i = \Theta_{m,d}\left(\eta_E \sigma_W \sqrt{d + \frac{2d(d-1)}{\pi m})}\right)$
    \item $\bar{\delta}_W^i = \Theta_{m,d}\left(\eta_W \sigma_E \sqrt{d + \frac{\alpha_i^2}{\bar\alpha^2}\frac{2d(d-1)}{\pi m})}\right)$
\end{itemize}}

\begin{proof}

The proof is straightforward by applying \cref{thm:indep} and \cref{thm:sign_hetero} to $\delta_E^i$ and $\delta_W^i$.
    
\end{proof}

\subsection{Proof of \cref{thm:large_vocab_regime}}
The formal statement of \cref{thm:large_vocab_regime} includes \cref{assumption:random_signal_at_initialization} as well. Actually, \cref{thm:large_vocab_regime} is a corollary of \cref{thm:signsgd_asymptotics} and \cref{lemma:alpha_asymptotics} and the proof is straightforward from the two results.

\subsection{Technical results for SignSGD dynamics}\label{sec:technical_results}

\subsubsection{Technical Result for $\delta_E^i$}
Fix two positive integers \(d,m\), and consider a $d\times m$ matrix $W$
\[
  W=\begin{bmatrix}W_{1}\\[-2pt]\vdots\\[-2pt]W_{d}\end{bmatrix}\in\mathbb R^{d\times m},
  \qquad
  W_{j}\stackrel{\text{i.i.d.}}{\sim}\mathcal N(0,\sigma_W^2 I_{m}).
\]
Let \(v\in\mathbb R^{m}\) be a \emph{random} vector with i.i.d.\ \(\mathcal N(0,1)\) coordinates, independent of \(W\).  
Define
\[
  M_{j}\;=\;\langle v,W_{j}\rangle,\quad
  S_{j}\;=\;\operatorname{sign}(M_{j})\in\{\pm1\},\quad
  X\;= S\, W = \;\sum_{j=1}^{d}S_{j}\,W_{j}\in\mathbb R^{1\times m}.
\]

We are interested in characterizing the expectation and covariance structure of the matrix $X$. The following standard result will be useful for this purpose.
\begin{lemma}[Stein’s lemma for a sign–Gaussian product]\label{lem:stein}
Let \((Z,G)\) be a bivariate centred Gaussian vector with
\(\operatorname{Var}(Z)=\operatorname{Var}(G)=1\)
and correlation \(\rho=\mathbb E[ZG]\).
Then
\[
  \boxed{\;
    \mathbb E\bigl[\operatorname{sign}(Z)\,G\bigr]
      \;=\;\sqrt{\frac{2}{\pi}}\;\rho.
  \;}
\]
\end{lemma}

\begin{proof}
Because \((Z,G)\) is joint Gaussian,
\((Z,G)=(Z,\rho Z+\sqrt{1-\rho^{2}}\,G')\)
with
\(G'\sim\mathcal N(0,1)\) independent of \(Z\).
Write \(\varphi(z)=\frac{1}{\sqrt{2\pi}}e^{-z^{2}/2}\) for the standard Gaussian
density.  Using Fubini (iterated integrals) and the symmetry of \(\varphi\),
\[
  \mathbb E\bigl[\operatorname{sign}(Z)\,G\bigr]
   =\rho\,\mathbb E\bigl[\lvert Z\rvert\bigr]
   =\rho\int_{\mathbb R}|z|\varphi(z)\,dz
   =\rho\sqrt{\frac{2}{\pi}}.\qedhere
\]
\end{proof}

The factor \(\sqrt{2/\pi}\) is simply the mean of \(|Z|\) for
\(Z\sim\mathcal N(0,1)\).

Now, we use this result to obtain analytic formulas for the mean and covariance of X.

\bigskip
\begin{thm}[Mean and covariance for independent \(v\)]\label{thm:indep}
With the setup above, we have
\[
  \boxed{\;
    \mathbb E[X]=0,\qquad
    \operatorname{Cov}(X)
      = d\,I_{m}
        +\frac{2}{\pi m}\,d(d-1)\,I_{m}.
  \;}
\]
\end{thm}

\begin{proof}
The proof is divided into three steps. We first obtain conditional results on $v$, then derive the mean and covariance in steps 2 and 3 by averaging over $W$.\\

\paragraph{\textbf{Step 1: conditional moments given \(v\).}}
\begin{enumerate}
  \item \emph{Conditional mean of one summand.}  
        Because \(W_{j}\sim\mathcal N(0,I_{m})\) and \(M_{j}=\langle v,W_{j}\rangle\), for any $i \in [m]$ 
        the pair \((M_{j},W_{ji})\) is jointly Gaussian with
        correlation
        \(
          \rho:=\tfrac{v_i}{\lVert v\rVert}.
        \)
        Stein’s lemma above gives
        \[
          \mathbb E_{W}[\,\operatorname{sign}(M_{j})\,W_{j}\mid v]
            \;=\;
            \sqrt{\frac{2}{\pi}}\;\frac{v}{\lVert v\rVert}.
        \]
  \item \emph{Conditional covariance of one summand.}  
        Still conditional on \(v\),
        \(\operatorname{sign}(M_{j})\) is \(\pm1\) with equal probability,
        independent of the magnitude of \(W_{j}\).  A direct computation gives
        \(\operatorname{Cov}_{W}(S_{j}W_{j}\mid v)=I_{m} - \frac{2}{\pi} \frac{v}{\|v\|} \frac{v^\top}{\|v\|} \).
\end{enumerate}
Because the rows \(W_{1},\dots,W_{d}\) are independent,
\begin{equation}\label{eq:cond-v}
  \mathbb E[X\mid v]
     = d\,\sqrt{\tfrac{2}{\pi}}\;\frac{v}{\lVert v\rVert},
  \qquad
  \operatorname{Cov}(X\mid v)=d\,\left(I_{m}- \frac{2}{\pi} \frac{v}{\|v\|} \frac{v^\top}{\|v\|}\right).
\end{equation}

\smallskip
\paragraph{\textbf{Step 2: unconditional mean.}}
The random vector \(v/\lVert v\rVert\) is \emph{uniform} on the sphere
\(S^{m-1}\), hence symmetric about the origin.  Taking expectations in
\eqref{eq:cond-v} yields
\[
  \mathbb E[X]=\mathbb E_{v}\bigl[\mathbb E[X\mid v]\bigr]
             = d\,\sqrt{\tfrac{2}{\pi}}\;
               \mathbb E\!\Bigl[\frac{v}{\lVert v\rVert}\Bigr]
             = 0.
\]

\smallskip
\paragraph{\textbf{Step3 : unconditional covariance.}}
Write $Y_{j}:=S_{j}W_{j}$, so that $X=\sum_{j=1}^{d}Y_{j}$.  
For every fixed $v$ we have, from Step 1,
\[
  \mathbb{E}[Y_{j}\mid v]=\mu(v)
      :=\sqrt{\frac{2}{\pi}}\;\frac{v}{\lVert v\rVert},
  \qquad
  \operatorname{Cov}(Y_{j}\mid v)
      =I_{m}-\mu(v)\mu(v)^{\!\top}.
\]
By applying the \emph{law of total variance} we can write the covariance of $X$ as follows\\
\[\operatorname{Cov}(X)=\mathbb{E}_{v}[\operatorname{Cov}(X\mid v)]
      +\operatorname{Cov}_{v}(\mathbb{E}[X\mid v]).\]

\begin{enumerate}
  \item \emph{First term.}
        Because the $Y_{j}$’s are conditionally independent,
        \[
          \mathbb{E}_{v}\bigl[\operatorname{Cov}(X\mid v)\bigr]
            =\mathbb{E}_{v}\Bigl[\sum_{j=1}^{d}\operatorname{Cov}(Y_{j}\mid v)\Bigr]
            =d\,\Bigl(I_{m}
                     -\underbrace{\mathbb{E}_{v}\!\bigl[\mu(v)\mu(v)^{\!\top}\bigr]}
                              _{=\;\frac{2}{\pi m}I_{m}}\Bigr)
            =\Bigl(d-\frac{2d}{\pi m}\Bigr)I_{m}.
        \]
        (We used $\mathbb{E}_{v}[\,v/\lVert v\rVert\;v^{\!\top}\!/\lVert v\rVert]
        =I_{m}/m$ by rotational invariance.)
  \item \emph{Sign–correlation term.}
        From Step1, $\mathbb{E}[X\mid v]=d\,\mu(v)$, so
        \[
          \operatorname{Cov}_{v}\bigl(\mathbb{E}[X\mid v]\bigr)
             =d^{2}\,\operatorname{Cov}_{v}\!\bigl(\mu(v)\bigr)
             =d^{2}\,\frac{2}{\pi}\,
               \operatorname{Cov}_{v}\!\Bigl(\frac{v}{\lVert v\rVert}\Bigr)
             =\frac{2d^{2}}{\pi m}\,I_{m}.
        \]
\end{enumerate}

Adding the two contributions gives
\[
  \operatorname{Cov}(X)
     =\Bigl[d-\tfrac{2d}{\pi m}
            +\tfrac{2d^{2}}{\pi m}\Bigr]I_{m}
     \;=\;\Bigl[d+\tfrac{2}{\pi m}\,d(d-1)\Bigr]I_{m},
\]
which is the covariance stated in the theorem.
\end{proof}

\subsubsection{Technical Result for $\delta_W^i$}
The previous technical lemma is used to derive the asymptotic behavior of $\delta_E^i$. Here, we prove a variant that will be useful for $\delta_W^i$.
\begin{thm}[Second moment of $\,X = E_i\,S(E^{\top} M)$]%
\label{thm:sign_hetero}
Fix integers $m,d\ge 1$ and independent random matrices
\[
  E=(E_{jk})_{1\le j\le m,\,1\le k\le d}
     \ \stackrel{\text{i.i.d.}}{\sim}\ \mathcal N(0,1),
  \qquad
  M=(M_{ja})_{1\le j,a\le m},
  \quad
  M_{ja}\stackrel{\text{i.i.d.\ across }a}{\sim}\mathcal N(0,\alpha_j^{2}),
\]
where the \emph{row variances} $\alpha_1^{2},\dots,\alpha_m^{2}>0$ are
deterministic and $M$ is independent of $E$.
For a fixed row--index $i\in\{1,\dots,m\}$ set
\[
  S \;=\; \operatorname{sign}\!\bigl(E^{\top}M\bigr)\in\{\pm1\}^{d\times m},
  \qquad
  X \;:=\; E_i\,S \;\in\mathbb R^{m}.
\]

Write  
\(\displaystyle
  \bar\alpha^{2}\;:=\;\frac1m\sum_{k=1}^{m}\alpha_k^{2}
\)
for the mean row–variance of $M$.  
Then every coordinate of \(X\) has mean zero and

\[
  \boxed{\;
    \operatorname{Var}(X_k)
      \;=\;
      d
      \;+\;
      \frac{2}{\pi}\;
      \frac{\alpha_i^{2}}{\bar\alpha^{2}}\;
      \frac{d(d-1)}{m},
      \qquad 1\le k\le m.
  \;}
\]

More precisely,
\(\displaystyle
  \operatorname{Cov}(X)
      =\bigl[d+\tfrac{2}{\pi}\tfrac{\alpha_i^{2}}{\bar\alpha^{2}}
              \tfrac{d(d-1)}{\,m}\bigr]\,I_m.
\)

\end{thm}

\begin{proof}
Throughout the proof we use the Stein identity stated in \cref{lem:stein}:  
if $(Z,G)$ is a centred bivariate Gaussian with $\operatorname{Corr}(Z,G)=\rho$
then
\(
  \mathbb E[\operatorname{sign}(Z)\,G]=\sqrt{2/\pi}\,\rho.
\)

\smallskip
\paragraph{\textbf{1.  Zero Mean.}}
Conditioning on $E$, each entry of $E^{\top}M$ a centered Gaussian variable, hence
\(
  \mathbb E[S(E^\top M)\mid E]=0
\).
Therefore \(\mathbb E[X]=0\).

\smallskip
\paragraph{\textbf{2.  Second Moment.}}

Fix a column index $k$. We have

\[
X_k = \sum_{a=1}^d E_{ia} S(Z_{ak}),
\]
where,
\[
  Z_{ak}
    := (E^{\top}M)_{ak}
     = E_{ia}M_{ik} + W_{ak},
  \qquad
  W_{ak}:=\sum_{j\neq i}E_{ja}M_{jk}.
\]

Since $X_k$ has zero mean, we have 

\begin{align*}
\E[X_k^2] &= \E\left[ \sum_{a=1}^d E_{ia}^2 + 2 \sum_{a < a'} E_{ia}S(Z_{ak}) E_{ia'}S(Z_{a'k})\right]\\
&= d + 2 \sum_{a < a'} \E\left[E_{ia}S(Z_{ak}) E_{ia'}S(Z_{a'k})\right].
\end{align*}

Fix some $a, a' \in \{1, 2, \dots, d\}$ such that $a \neq a'$. Let $\Sigma = \sigma(M)$ be the sigma-Algebra generated by $M$, and denote $S_{ak} = S(Z_{ak})$ and the same for $a'$.

Given $\Sigma$, the pair $(Z_{ak},E_{ia})$
is Gaussian with correlation
\[
  \rho_k= \textrm{Corr}(Z_{ak}, E_{ia} \mid \Sigma) = \frac{M_{ik}}{\sqrt{\sum_{j=1}^{m}M_{jk}^{2}}}
\]
Using Stein's identity,
\[
  \mu\;:=\;\mathbb E\bigl[E_{ia}\,S_{ak} \mid \Sigma\bigr]
     =\sqrt{\frac{2}{\pi}}\;\rho_k.
\]

The same holds for $a'$ and we have $E\bigl[E_{ia'}\,S_{a'k} \mid \Sigma\bigr] = \sqrt{\frac{2}{\pi}}\;\rho_k$.

Conditionally on $\Sigma$, the random variables $E_{ia} S_{ak}$ and $E_{ia'} S_{a'k}$ are independent, which yields

\[
  \mathbb E\bigl[E_{ia}S_{ak}\,E_{ia'}S_{a'k}\mid  \Sigma\bigr]
     =\frac{2}{\pi}\,\rho_k^{2}.
\]

Therefore, 
\[
  \mathbb E\bigl[E_{ia}S_{ak}\,E_{ia'}S_{a'k} \bigr]
     =\frac{2}{\pi}\,\E \rho_k^{2}.
\]

Denote $M_{j\ell} = \alpha_{i} g_{j\ell}$ for all  $j, \ell$ where $g_{j\ell} \sim \normal(0,1)$ iid. Then 
$$
\E \rho_k^2 = \alpha_i^2 \times \E \left[\frac{g_{ik}^2}{ \sum_{j=1}^m \alpha_j^2 g_{jk}^2} \right].
$$

The term $\E \left[\frac{g_{ik}^2}{ \sum_{j=1}^m \alpha_j^2 g_{jk}^2} \right]$ is invariant to permutations of the vector $(g_{1k}^2, \dots, g_{km}^2)$ which implies that $\E \rho_k^2$ is proportional to $\alpha_i^2$. By summing over $i$ we conclude that 
$$
\E \rho_k^2 = \frac{\alpha_i^2}{ \sum_{j=1}^m \alpha_j^2}.
$$

Therefore,
\[
  \mathbb E\bigl[E_{ia}S_{ak}\,E_{ia'}S_{a'k} \bigr]
     =\frac{2}{\pi}\, \frac{\alpha_i^2}{m\,  \bar{\alpha}^2}.
\]

Since there are \(d(d-1)\) off‑diagonal index pairs \((a,b)\) for a fixed
column \(k\), the off‑diagonal sum contributes

\[
  \frac{2}{\pi}\,
  \frac{\alpha_i^{2}}{\bar\alpha^{2}}\,
  \frac{d(d-1)}{m}.
\]

\smallskip
\paragraph{\textbf{3.  Independence across columns.}}
Different columns of \(M\) (different \(k\)) are independent, so
\(\operatorname{Cov}(X)\) is diagonal and each coordinate variance is the same as above.
\end{proof}

\newpage

\section{Additional Empirical Details}\label{app:emp}
\subsection{Experimental Setup for \cref{fig:no_hp_transfer_with_mup}}

\begin{itemize}
    \item Model arch: embedding layer, 2 hidden layers, projection layer
    \item Width: $d \in \{2^k, k=8, \dots, 12\}$
    \item Attention head dimension: 64
    \item MLP: Up\_proj, Down\_proj, with GeLU activation in between.
    \item Vocabulary size: 32768 (BPE Tokenizer)
    \item Training dataset: Wikitext2
    \item Batch size: 256
    \item Sequence length: 256
    \item Algorithm: Adam with constant schedule ($\eta_E$ for embed, $\eta_H$ for hidden, $\eta_{W}$ for projection)
    \item Initialization: Kaiming init ($\sigma_E$ for embed, $\sigma_H = 1/\sqrt{d}$ for hidden, $\sigma_W$)
    \item Steps: $10^4$ 
\end{itemize}

We sweep over $\eta_E$, $\eta_H$, $\eta_W$, $\sigma_E$, and $\sigma_W$ with 2 random seeds. The runs took 2 weeks to complete on a 4xGH200 node.

\subsection{Experimental Setup for \cref{fig:emb_lr_vocab_scaling}}

\begin{itemize}
    \item Model arch: embedding layer, 2 hidden layers, projection layer
    \item Width/Vocab: $(d,m) \in \{(2^k, 2^{k+3}), k=8, \dots, 12\}$
    \item Attention head dimension: 64
    \item MLP: Up\_proj, Down\_proj, with GeLU activation in between.
    \item Training dataset: Wikitext2
    \item Batch size: 256
    \item Sequence length: 256
    \item Algorithm: Adam with constant schedule ($\eta_E$ for embed, $\eta_H=0.2/d$ for hidden, $\eta_{W}=0.2/d$ for projection)
    \item Initialization: Kaiming init $\sigma_E=\sigma_W =\sigma_H = 1/\sqrt{d}$
    \item Steps: $10^4$ 
\end{itemize}

We sweep only over $\eta_E$ with 3 random seeds (with a more finegraind grid compared to the setup above). The runs took 4 days to complete on a 4xGH200 node.

\subsection{Experimental setup for 1B model pretraining}

\begin{center}
  {
\small
\begin{tabular}{l c}
\toprule
Parameter  & Value\\
\midrule
\multicolumn{2}{c}{General}\\
\midrule
\texttt{vocab\_size} & 32064\\
% \texttt{pretrain\_context\_len} & 8192\\
\texttt{n\_position} & 4096\\
\texttt{n\_layers} & 24\\
\texttt{n\_embed} & 2048\\
\texttt{normalization} & LayerNorm \\

\midrule
\multicolumn{2}{c}{Attention specific}\\
\midrule
\texttt{window\_size} & 2048\\
\texttt{n\_head} & 32\\
\texttt{n\_kv\_head} & 8\\
\texttt{head\_dim} & 128\\
\texttt{rotary\_dim} & 128 \\
\midrule
\multicolumn{2}{c}{FFN specific}\\
\midrule
\texttt{activation} & SwiGLU\\
\texttt{inner\_dim} & 8192\\

\bottomrule
\end{tabular}
}
\vspace{.2cm}
% \caption{Model architecture parameters.}
\label{tab:model_params}
% \end{table}

\captionof{table}{1B LLM Specifics.}
% \end{minipage}  
\end{center}

\end{document}